\documentclass{article}
\usepackage{spconf}
\usepackage{graphicx} 
\usepackage{booktabs} 
\usepackage{amsmath}
\usepackage{amsfonts}
\usepackage{amssymb}
\usepackage{epstopdf}
\usepackage{dblfloatfix}
\usepackage[caption=false,font=footnotesize]{subfig}
\usepackage{color}
\usepackage{amsthm}
\usepackage{cite}

\ninept

\newtheorem{theorem}{Theorem}

\newtheorem{definition}{Condition}

\title{Distributed Estimation of the Operating State of a Single-Bus DC MicroGrid without an External Communication Interface}
\name{Marko Angjelichinoski$^{*}$, Anna Scaglione$^{\dagger}$, Petar Popovski$^{*}$ and \v Cedomir Stefanovi\'c$^{*}$\thanks{The work presented in this paper was supported in part by EU, under grant agreement no. 607774 ``ADVANTAGE''.}}
\address{{\small$^{*}$Dept. of Electronic Systems, Aalborg Univ., Denmark, $^{\dagger}$School of ECEE, Arizona State Univ., AZ, USA}\\{\small Emails: \texttt{maa@es.aau.dk,anna.scaglione@asu.edu,petarp@es.aau.dk,cs@es.aau.dk} }}
\begin{document}
%
\maketitle
\begin{abstract}
We propose a decentralized Maximum Likelihood solution for estimating the stochastic renewable power generation and demand in single bus Direct Current (DC) MicroGrids (MGs), with high penetration of droop controlled power electronic converters.
The solution relies on the fact that the primary control parameters are set in accordance with the local power generation status of the generators.
Therefore, the steady state voltage is inherently dependent on the generation capacities and the load, through a non-linear parametric model, which can be estimated.
To have a well conditioned estimation problem, our solution avoids the use of an external communication interface and utilizes controlled voltage disturbances to perform distributed training. 
Using this tool, we develop an efficient, decentralized Maximum Likelihood Estimator (MLE) and formulate the sufficient condition for the existence of the globally optimal solution.
The numerical results illustrate the promising performance of our MLE algorithm.
\end{abstract}
\section{Introduction}
\label{sec:intro}

Low Voltage Direct Current (LVDC) MicroGrids (MGs) are gaining popularity due to the flexibility of the control, the absence of reactive power component and the easy integration with emerging renewable generation technologies \cite{ref1,ref2,ref3}.
They use \emph{power electronic converters} to interface the Distributed Energy Resources (DERs) with the LVDC distribution infrastructure.
The converters implement a set of \emph{control} algorithms, organized in a hierarchy that consists of \emph{primary}, \emph{secondary} and \emph{tertiary} levels \cite{ref2,ref4,ref5,ref6}.
The primary controller regulates the steady state bus voltage, keeping the balance between the supplied power and the load demand.
It is commonly implemented via the \emph{droop} control law in decentralized configuration, where each controller uses only the local output current to control the voltage as the load varies \cite{ref2,ref4,ref6}.
On the other hand, the secondary/tertiary controllers, which perform 
various system optimization procedures, require regular updates of the power generation status of remote DERs and the current load \cite{ref2}.
Traditionally, a communication network is used to send those updates, which increases the complexity and the implementation cost of the MG system, as well as makes its reliability dependent on an external system \cite{ref1,ref2,ref3}.

Motivated by the shortcomings of external communication systems, in this paper we propose a novel framework for single-bus DC MGs, which enables DER units to estimate the power generation capacities of all other DERs, as well as the load power demand by relying solely on the capabilities of power electronic controllers.
The proposed framework enables DERs to learn the \emph{operating state} of the MG that can be used by various control applications, such as optimal economic dispatch \cite{ref7,ref8}, optimal power flow \cite{ref9} and market optimizations \cite{ref10}.
Since each DER learns the status of all remote units, the optimal control decisions can be made without any further coordination, allowing for a fully \emph{decentralized} control architecture.

The proposed solution reuses the existing primary control interface and it does not require any additional hardware and/or external communication support.
Its main principle of operation exploits the fact that the configuration of the \emph{primary droop controller} makes the steady state bus voltage functionally dependent on the variable generation capacities of the DERs and on the current value of the load.
Specifically, the \emph{virtual admittance} control parameter of each droop controller is nominally set to be proportional to the power capacity of the DER, while the feedback loop is closed via the output current of the unit that varies with changes in the load \cite{ref11,ref12,ref13}.
Thus, the steady state voltage can be described through a non-linear model, parametrized by the DER generation capacities and the load power demand.
To make the parameters of the model identifiable, the controllers, for a \emph{limited} time period, simultaneously switch between different operating points of the droop control following predetermined patterns called \emph{training sequences}.
These sequences are \emph{a priori} known to all controllers and they cause deviations of the steady state voltage, which are observed locally and enable each DER to apply decentralized Maximum Likelihood Estimation (MLE).
The specific contributions of this paper can be summarized as follows: (i) formulation of the estimation problem based on the model of the primary droop control, (ii) identification of the sufficient condition for unique identifiability of the MG configuration (i.e., the states of the DERs and of the load), and (iii) solution to the MLE problem.
We also illustrate the potential of the proposed framework in an example MG system.

The rest of the paper is organized as follows.
This section is concluded with a brief review of the related work.
Section \ref{sec:DCMGModel} introduces the system model and formulates the problem of estimating renewable generation and the load.
Section \ref{sec:MLE} presents the MLE as the main result of the paper.
Section \ref{sec:Results} presents the numerical results and section \ref{sec:conc} briefly discusses our ongoing work on the topic.

\textbf{Related Work}: Switching between different operating points of the converters was previously used for active impedance estimation for the Thevenin equivalent model\cite{ref14,ref15,ref16}.
However, this application is confined to very simple scenarios where only a single converter disturbs the state of the bus and only one parameter gets estimated.
In our framework, multiple controllers deviate primary control parameters in order to ``train'' the system simultaneously, extracting significantly more information about the system state.
Finally, we note that the proposed framework is in line of recent works in MGs \cite{ref1,ref2,ref17,ref18,ref19,ref20,ref21,ref22} that suggest to avoid installing a separate cyber infrastructure to support the control architecture, due to reliability/availability concerns and installation costs.
Instead, recent advances advocate to use the signal processing potential residing in the power electronic converter measurements and control circuits, allowing for a fully self-contained MG implementation \cite{ref20,ref21,ref22}.

\section{System Model and Problem Formulation}
\label{sec:DCMGModel}

\subsection{Primary Control and Steady State Characterization}

\begin{figure}[t]
\centering
\includegraphics[scale=0.3]{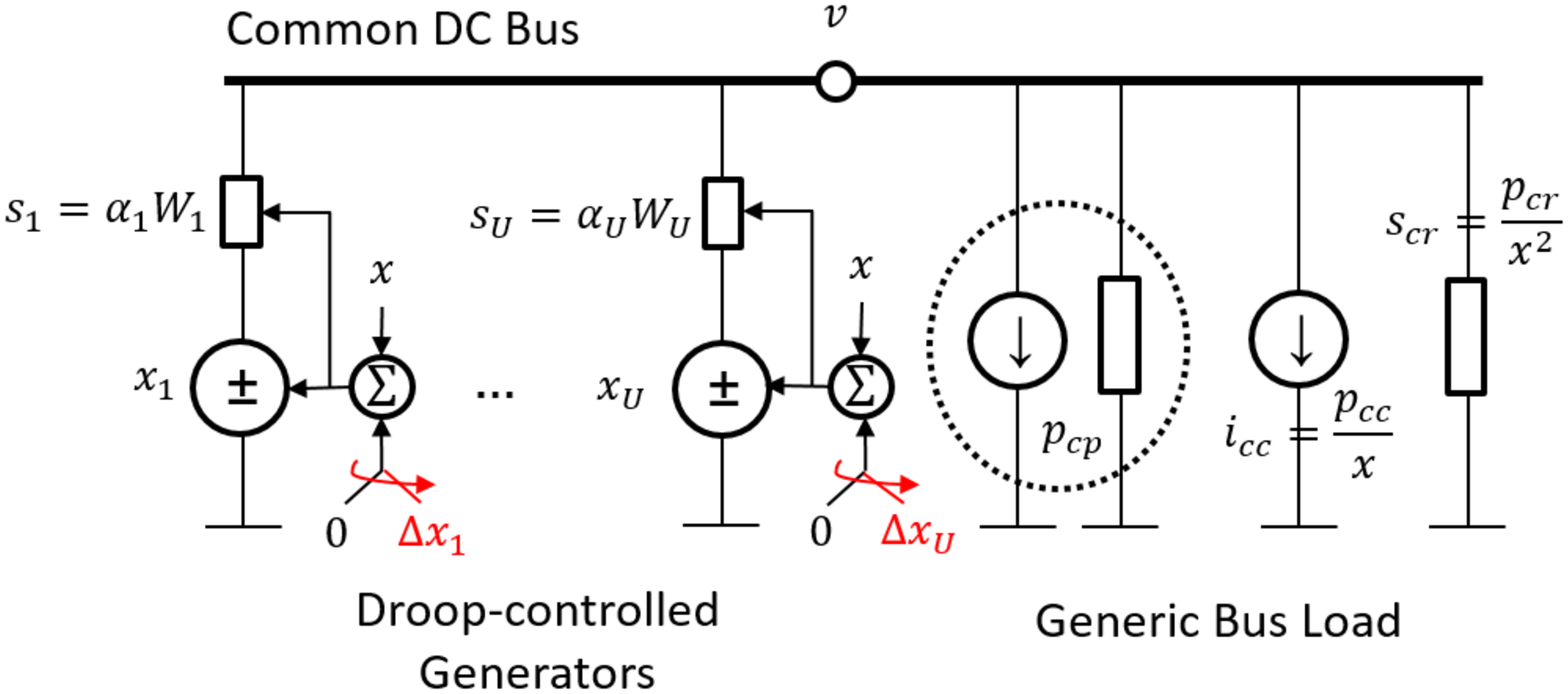}
\caption{Single bus DC MicroGrid in steady state.}
\label{SBMG}
\end{figure}

We focus on a single bus DC MG system (see Fig.~\ref{SBMG}), assuming that all units are connected to a common point, i.e., the \emph{bus}, described by a steady-state voltage $v$.
Note that this is a valid MG model for small, localized systems that span a limited geographical areas and the effect of the distribution lines can be neglected \cite{ref2,ref3,ref11,ref12,ref13}.
We assume that $U$ DER units are interfaced by the bus through power electronic converters.
The primary control executed by the converters is implemented via the \emph{droop control law} that has two controllable parameters: the \emph{reference voltage} and the \emph{virtual admittance}.
In Fig.~\ref{SBMG}, the droop controlled units are modeled as Voltage Source Converters that jointly regulate the bus voltage as follows \cite{ref11,ref12}:
\begin{equation}\label{eq:drooplaw}
	v = x_u + s_u^{-1}i_u,\; u=1,...,U,
\end{equation}
where $x_u$ and $s_u$ denote the reference voltage and virtual admittance, respectively, and $i_u$ is the output current of the unit, $u=1,...,U$.
The power generation ratings of the droop controlled DERs are denoted by $W_u,u=1,...,U$.
In standard applications for single bus systems, the reference voltages $x_u,u=1,...,U$ are set to be equal to the \emph{rated system voltage}, denoted by $x$, while the virtual admittances are set to enable proportional power sharing among the DERs based on their respective current/power ratings:
\begin{equation}\label{eq:virtualadm1}
	s_u  = \frac{i_{u,\max}}{x_u-v_{\min}} = \alpha_u W_u,\;u=1,...,U,	
\end{equation}
where $\alpha_u=((x_u-v_{\min})v_{\min})^{-1}$, $v_{\min}$ is the minimal bus voltage that the system is configured to tolerate, and $i_{u,\max} = {W_uv_{\min}^{-1}}$ is the current rating of the unit whose value corresponds to the power rating for proportional power sharing \cite{ref2,ref11,ref12,ref13}.
The aggregate load hosted by the bus is modeled through a constant admittance $s_{cr}=\frac{p_{cr}}{x^2}$, a constant current $i_{cc}=\frac{p_{cc}}{x}$ and a constant power $p_{cp}$ component; $p_{cr}$, $p_{cc}$ and $p_{cp}$ are the rated power consumptions at voltage $x$.

The behavior of the system in steady-state is governed by the Ohm's and Kirchoff's laws, which for the system shown in Fig.~\ref{SBMG}, under the described primary control configuration, produce the following current balance equation:
\begin{equation}\label{eq:current_balance}
	\sum_{u=1}^U(x_u-v)s_u -v\frac{p_{cr}}{x^2} - \frac{p_{cc}}{x} - \frac{p_{cp}}{v} = 0.
\end{equation}
Replacing $s_u$ with \eqref{eq:virtualadm1} and solving for $v$ yields the following unique solution \cite{ref12}:
\begin{align}\nonumber
	v & = \frac{\sqrt{(\sum_{u=1}^Ux_u\alpha_uW_u - \frac{p_{cc}}{x})^2-4p_{cp}(\sum_{u=1}^U\alpha_uW_u + \frac{p_{cr}}{x^2})}}{2(\sum_{u=1}^U\alpha_uW_u + \frac{p_{cr}}{x^2})}  +\\\label{eq:bus_voltage}
	  & + \frac{\sum_{u=1}^Ux_u\alpha_uW_u - \frac{p_{cc}}{x}}{2(\sum_{u=1}^U\alpha_uW_u + \frac{p_{cr}}{x^2})}.
\end{align}
We observe that the bus voltage $v$ is functionally dependent on the power generation capacities $W_u$, $u=1,...,U$, and the rated power consumptions $p_{cr}$, $p_{cc}$, $p_{cp}$ through the non-linear model \eqref{eq:bus_voltage}.

\subsection{Formulation of the Estimation Problem}

\subsubsection{Parameter Vector}

Every controller $k$, $k=1,...,U$, wants to estimate the generation capacities of the other droop controlled generators $W_u$, $u=1,...,U$, for $u\neq k$, as well as the demand of the bus load, i.e., $p_{cr}$, $p_{cc}$ and $p_{cp}$.
The vector of unknown parameters is:
\begin{align}\label{eq:param_general}
	\boldsymbol{\theta}_k & = [\mathbf{W}_k^T,\;\mathbf{p}_L^T]^T\in\mathbb{R}^{(U+2)\times 1},\;{k=1,...,U},\\\label{eq:Wparam}
	\mathbf{W}_k & = [W_u]_{u=1,...,U,u\neq k}^T\in\mathbb{R}^{(U-1)\times 1},\\\label{eq:plparam}
	\mathbf{p}_L & = [p_{cr},p_{cc},p_{cp}]^T\in\mathbb{R}^{3\times 1},
\end{align}
where $\mathbf{W}_k$ and $\mathbf{p}_L$ are the power generation and load demand vector, respectively.
Since the steady-state bus voltage is a function of $\boldsymbol{\theta}_k$, it is possible for controller $k$ to estimate $\boldsymbol{\theta}_k$ in a decentralized manner using local observations of $v$.
Because \eqref{eq:bus_voltage} is a mapping $v(\boldsymbol{\theta}_k):\mathbb{R}^{(U+2)}\mapsto\mathbb{R}$, the excitation of the state of the system is necessary to ensure \emph{identifiability} of $\boldsymbol{\theta}_k$.
This excitation comes in the form {of \emph{discrete-time training sequences}, embedded in the reference voltage control parameter, as elaborated in the following subsection}\footnote{Section \ref{sec:MLE} formulates the condition that the training sequences need to satisfy for unique identifiability of the parameter vector $\boldsymbol{\theta}_k,k=1,...,U$.}.

\subsubsection{Training Sequences and Measurement Vector}

The \emph{training period} is divided into {$N$ slots of duration $T_S$}.
The controllers are assumed to be {slot-} and training period- synchronized, i.e., their training sequences start at the same slot. 
In slot $n=1,...,N$, each controller changes its steady state operating point by applying small {deviations} on the reference voltage parameters (see Fig.~\ref{SBMG}):
\begin{equation}\label{eq:refvoltagtraining}
		x_{u}[n] = x + \Delta x_u[n],\;n=1,...,N,\;u=1,...,U,
\end{equation}
where the {deviations} $\Delta x_u[n]$ satisfy:
\begin{equation}
		{|\Delta x_u[n]|}\leq \delta x,\;n=1,...,N,\;u=1,...,U,
\end{equation}
and $\delta\ll 1$ is a small positive number, determined by the system application, that limits the amount of bus voltage {ripple}.
In other words, the controllers simultaneously inject training sequences of duration $N$, which are compactly denoted with the vectors $\Delta\mathbf{x}_u = [\Delta x_u[1],\hdots,\Delta x_u[N]]^T\in\mathbb{R}^{N\times 1},\;u=1,...,U$ {and are \emph{known} to \emph{all} controllers.}
Using \eqref{eq:refvoltagtraining} in \eqref{eq:virtualadm1}, the virtual resistances in slot $n$ are:
\begin{equation}\label{eq:virtualadm2}
	s_u[n]  = \frac{v_{\min}^{-1}W_u}{x + \Delta x_u[n] - v_{\min}} = \alpha_u[n] W_u,\;u=1,...,U.
\end{equation}
The reference voltage deviations lead to changes of the bus voltage which are observed and measured by the controllers:
\begin{equation}\label{eq:voltage_dev}
v[n;\boldsymbol{\theta}_k] = \overline{v}(\boldsymbol{\theta}_k) + \Delta v[n;\boldsymbol{\theta}_k],\;n=1,...,N,
\end{equation}
where $v[n;\boldsymbol{\theta}_k]$ is the {bus voltage in slot $n$. 
$\overline{v}(\boldsymbol{\theta}_k)$ is the bus voltage level in the absence of training, which can be calculated from \eqref{eq:bus_voltage} after replacing the reference voltages $x_u,u=1,...,U$ with the rated voltage $x$, while $\Delta v[n;\boldsymbol{\theta}_k]$ is the voltage deviation due to training in slot $n$.}
It is important to note that condition \eqref{eq:virtualadm2} guarantees that the output powers of the units will not violate the ratings $W_u$ and the bus voltage $v[n]$ will not drop below $v_{\min}$ as long as $\delta\leq 1 - xv_{\min}^{-1}$.

In the following, we {omit the explicit dependence of the bus voltage on $\boldsymbol{\theta}_k$} as it is clear from the context.
To obtain $v[n]$, controller $k$ samples the bus voltage with frequency $f_S$.
{The duration $T_S$ complies with the bandwidth limits of the primary control channel and allows the bus to reach a steady state in time $\tau$ where $0<\frac{\tau}{T_S} \ll 1$.}
The controller averages $(T_S-\tau)f_S$ bus voltage samples at the end of each slot, acquired in the steady state period, to obtain the noisy measurement:
\begin{equation}\label{eq:observation}
	\tilde{v}_k[n] = v[n] + z_k[n],\;k=1,...,U,\;u=1,...,U,
\end{equation}
where $z_k[n]$ is the noise term.
In vector notation:
\begin{align}\label{eq:obsvector}
\tilde{\mathbf{v}}_k & = \mathbf{v} + \mathbf{z}_k,\;k=1,...,U,
\end{align}
where $\tilde{\mathbf{v}}_k = [\tilde{v}_k[1],\hdots,\tilde{v}_k[N]]^T\in\mathbb{R}^{N\times 1}$, ${\mathbf{v}} = [{v}[1],\hdots,{v}[N]]^T$ $ \in\mathbb{R}^{N\times 1}$ and $\mathbf{z}_k = [z_k[1],\hdots,z_k[N]]^T\in\mathbb{R}^{N\times 1}$.
The central problem is to obtain an estimate of $\boldsymbol{\theta}_k$, denoted with $\hat{\boldsymbol{\theta}}_k$, using only $\tilde{\mathbf{v}}_k$ and the non-liner parametric model \eqref{eq:bus_voltage}.

\section{Maximum Likelihood Estimator}
\label{sec:MLE}

As each voltage entry in $\tilde{\mathbf{v}}_k$ is obtained by averaging, the noise is assumed to {follow Gaussian distribution}, $\mathbf{z}_k\sim\mathcal{N}(\mathbf{0},\sigma^2\mathbf{I}_N)$ \cite{ref23}.
The log-likelihood function of $\boldsymbol{\theta}_k$, for a given $\tilde{\mathbf{v}}_k$ is:
\begin{equation}\label{eq:loglikelihood}
	\mathcal{L}(\boldsymbol{\theta}_k|\tilde{\mathbf{v}}_k) \sim -(\tilde{\mathbf{v}}_k - {\mathbf{v}}_k)^T(\tilde{\mathbf{v}}_k - {\mathbf{v}}_k) \geq 0.
\end{equation}
The Maximum Likelihood Estimator (MLE) is defined as:
\begin{equation}\label{eq:MLE_general}
	\hat{\boldsymbol{\theta}}_{k,ML} = \min_{\boldsymbol{\theta}_k}\left\{-\mathcal{L}(\boldsymbol{\theta}_k|\tilde{\mathbf{v}}_k)\right\}.
\end{equation}
To avoid solving a non-linear optimization problem, with possibly non-convex objective function over high-dimensional parameter space, we establish Theorem \ref{theorem1}.
For this purpose, we introduce the following matrix:
\begin{align}\label{eq:elemH}
	\tilde{\mathbf{H}}_k & = [\tilde{\mathbf{\Pi}}_k,\;\tilde{\mathbf{\Xi}}_k]\in\mathbb{R}^{N\times(U+2)},\;k=1,...,U
\end{align}
where
\begin{align}\label{eq:elemPi}
	\tilde{\mathbf{\Pi}}_k & = \text{diag}(\tilde{\mathbf{v}}_k)\tilde{\mathbf{\Delta}}_k\in\mathbb{R}^{N\times(U-1)},\\
	\tilde{\mathbf{\Delta}}_k & = [\boldsymbol{\alpha}_u\circ(\tilde{\mathbf{v}}_k - x\mathbf{1}_N - \Delta\mathbf{x}_u)]_{u=1,...,U,u\neq k} \in\mathbb{R}^{N\times(U-1)}, \\\label{eq:elemalpha}
	\boldsymbol{\alpha}_u & = [\alpha_u[1],\hdots,\alpha_u[N]]^T\in\mathbb{R}^{N\times 1}\\\label{eq:elemXi}
	\tilde{\mathbf{\Xi}}_k & = {[x^{-2}\tilde{\mathbf{v}}_k\circ\tilde{\mathbf{v}}_k,\;x^{-1}\tilde{\mathbf{v}}_k,\;\mathbf{1}_N]} \in\mathbb{R}^{N\times 3},
\end{align}
where $\circ$ denotes the Hadamard product.
Then, we establish the following condition.

\begin{definition}({Sufficient Excitation}).
The matrix $\tilde{\mathbf{H}}_k$ has full column rank, i.e., $\text{\emph{rank}}(\tilde{\mathbf{H}}_k)=U+2$.
\end{definition}

We are now ready to state the following result:
\begin{theorem}\label{theorem1}
If Condition 1 holds, then the unique global minimizer of \eqref{eq:MLE_general} is the following Least Squares solution:
\begin{equation}\label{eq:MLE_sol}
	 \hat{\boldsymbol{\theta}}_{k,ML} = (\tilde{\mathbf{H}}_k^T\tilde{\mathbf{H}}_k)^{-1}\tilde{\mathbf{H}}_k^T\tilde{\boldsymbol{\pi}}_kW_k,
\end{equation}
where
\begin{align}\label{eq:elempi}
\tilde{\boldsymbol{\pi}}_k = -\text{diag}(\tilde{\mathbf{v}}_k)(\boldsymbol{\alpha}_k\circ(\tilde{\mathbf{v}}_k - x\mathbf{1}_N - \Delta\mathbf{x}_k))\in\mathbb{R}^{N\times 1},
\end{align}
and $\boldsymbol{\theta}_k$ and $\tilde{\mathbf{H}}_k$ are given with \eqref{eq:param_general} and \eqref{eq:elemH}, respectively.
\end{theorem}
\begin{proof}
Any solution of the MLE problem \eqref{eq:MLE_general} must satisfy the stationary point condition:
\begin{equation}\label{proof_step1}
\frac{\partial \mathcal{L}(\boldsymbol{\theta}_k|\tilde{\mathbf{v}}_k)}{\partial\boldsymbol{\theta}_k} = -\frac{1}{\sigma^2}(\tilde{\mathbf{v}}_k - {\mathbf{v}}_k)^T\frac{\partial \mathbf{v}_k}{\partial\boldsymbol{\theta}_k}=\mathbf{0}.
\end{equation}
The trivial solution to \eqref{proof_step1} is $\tilde{\mathbf{v}}_k - {\mathbf{v}}_k=\mathbf{0}$, or $v[n]=\tilde{v}_k[n],n=1,...,N$.
Plugging the solution in \eqref{eq:current_balance} and multiplying on both sides with $\tilde{v}_k[n]\neq 0$, gives the power balance equation:
\begin{equation}\label{proof_step2}
\tilde{v}_k^2[n](\sum_{u=1}^U\alpha_u[n]W_u + \frac{p_{cr}}{x^2}) - \tilde{v}_k[n](\sum_{u=1}^U\alpha_u[n]x_u[n]W_u - \frac{p_{cc}}{x}) + p_{cp} = 0.
\end{equation}
Reorganizing \eqref{proof_step2} in matrix form, produces the linear system:
\begin{equation}\label{proof_step3}
\tilde{\mathbf{H}}_k\boldsymbol{\theta}_k = \tilde{\boldsymbol{\pi}}_kW_k,
\end{equation}
where $\tilde{\mathbf{H}}_k$ and $\tilde{\boldsymbol{\pi}}_k$ are given with \eqref{eq:elemH} and \eqref{eq:elempi}, respectively.
If $(\tilde{\mathbf{H}}_k^T\tilde{\mathbf{H}}_k)^{-1}$ exists, then the unique solution to \eqref{proof_step3} in least squares sense is \eqref{eq:MLE_sol}.
Moreover, \eqref{eq:MLE_sol} is the global minimizer of \eqref{eq:MLE_general} since $\mathcal{L}(\boldsymbol{\theta}_k|\tilde{\mathbf{v}}_k)=0$ when $\mathbf{v}_k=\tilde{\mathbf{v}}_k$.
\end{proof}

In practice, the sufficient excitation condition will hold if the training sequences are properly designed.
Specifically, they should satisfy: 1) $N\geq U+2$, i.e., the duration of the excitation should be at least as long as the number of estimated parameters, and, 2) the sequences should be linearly independent.
The linear independence can be achieved by choosing the training sequences from a well known orthogonal codes such as Walsh, Gold, Hadamard etc.

\begin{figure*}[t]
\centering
\includegraphics[scale=0.47]{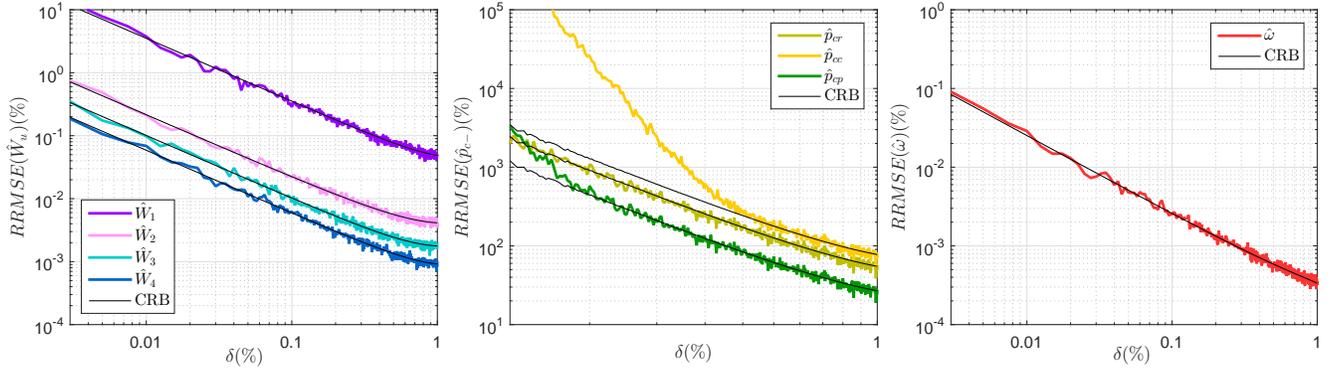}\hfill
\caption{{Performance of the MLE for generation capacity and load estimation.}}
\label{results}
\end{figure*}

\textbf{Estimating the Load Consumption}:
In many monitoring and control applications, detailed knowledge of the load components is not necessary, i.e., the aggregate information about the total power demand is sufficient \cite{ref7,ref10}.
Therefore, at this point we consider estimating only the total load demand $p_L=p_{cr}+p_{cc}+p_{cp}$.
Multiplying both sides of \eqref{eq:current_balance} by $v[n]$ and replacing $v[n]$ with \eqref{eq:voltage_dev} yields:
\begin{align}\label{proof_step21}
& {v}^2[n]\sum_{u=1}^U\alpha_u[n]W_u - {v}[n]\sum_{u=1}^U\alpha_u[n]x_u[n]W_u + p_L^{\star}[n] = 0,\\\label{proof_step22}
 & p_L^{\star}[n] =\omega + \chi\Delta v[n] + \zeta\Delta v^2[n],\\\label{proof_step23}
 & \omega  = \frac{\overline{v}^2}{x^2}p_{cr}+\frac{\overline{v}}{x}p_{cc}+p_{cp},\;\chi = 2\overline{v}\frac{p_{cr}}{x^2}+\frac{p_{cc}}{x},\;\zeta=\frac{p_{cr}}{x^2},
\end{align}
where $p_L^{\star}[n]$ is the total power consumed by the load in slot $n$, consisting of (i) $\omega$, which is the total power consumed at bus voltage $\overline{v}$, and (ii) the resistive and constant current components described by parameters $\chi$ and $\zeta$, respectively.
In practice, $\overline{v}$ is tightly regulated around the rated voltage $x$ \cite{ref6,ref2}, which makes $\omega$ a good approximation of the total power demand, i.e, $\omega\approx p_L$.
We define the transformed parameter vector:
\begin{align}\label{eq:param_vector2}
\boldsymbol{\theta}_k^{\star} & = [\mathbf{W}_k^T,\;\mathbf{p}_L^{\star T}]^T\in\mathbb{R}^{(U+2)\times 1},\;k=1,...,U,\\\label{eq:load_vector}
\mathbf{p}_L^{\star} & = [\omega,\;\chi,\;\zeta]^T\in\mathbb{R}^{3\times 1},
\end{align}
where $\mathbf{p}_L^{\star}$ is the power consumption vector and $\mathbf{W}_k$ is given with \eqref{eq:Wparam}; the map $\boldsymbol{\theta}_k^{\star}(\boldsymbol{\theta}_k):\mathbb{R}^{(U+2)\times 1}\mapsto\mathbb{R}^{(U+2)\times 1}$ is bijection.
Theorem \ref{theorem1} can then be restated for $\boldsymbol{\theta}_k^{\star}$, using \eqref{proof_step21} with the matrix $\tilde{\mathbf{H}}_k$ obtaining the following form:
\begin{equation}
	\tilde{\mathbf{H}}_k^{\star} = [\tilde{\mathbf{\Pi}}_k,\;\mathbf{1}_N,\;\Delta\tilde{\mathbf{v}}_k,\;\Delta\tilde{\mathbf{v}}_k\circ\Delta\tilde{\mathbf{v}}_k]\in\mathbb{R}^{N\times(U+2)},
\end{equation}
where $\Delta\tilde{\mathbf{v}}_k = \tilde{\mathbf{v}}_k - \overline{v}\mathbf{1}_N$.

\textbf{Estimation Error}:
To characterize the statistical performance of the MLE \eqref{eq:MLE_sol}, we evaluate the Cramer-Rao Bound (CRB) on the Mean Squared Error (MSE) matrices $\textbf{MSE}(\hat{\boldsymbol{\theta}}_k)$ and $\textbf{MSE}(\hat{\boldsymbol{\theta}}_k^{\star})$ \cite{ref24}.
Omitting the derivation, we state only the final result:
\begin{align}\label{eq:CRB1}
 \textbf{{MSE}}(\hat{\boldsymbol{\theta}}_k) \succeq \bigg(\sum_{n=1}^N\frac{\sigma^2}{\lambda^2[n]}\mathbf{q}_k[n]\mathbf{q}_k^T[n]\bigg)^{-1},
\end{align}
where the scalar $\lambda[n]$ and $\mathbf{q}_k[n]\in\mathbb{R}^{(U+2)\times 1}$ can be calculated as:
\begin{align}\nonumber
\lambda[n] & =\sum_{u=1}^U(2v[n] - x_u[n])\alpha_u[n]W_u + 2v[n]\frac{p_{cr}}{x^2}-\frac{p_{cc}}{x},\\\nonumber
\mathbf{q}_k[n] & = \bigg[...,\alpha_u[n]v[n](v[n]-x_u[n]),...,\frac{v[n]^2}{x^2},\frac{v[n]}{x},1\bigg]_{u\neq k}^T.
\end{align}
Similarly, the MSE matrix of $\hat{\boldsymbol{\theta}}_k^{\star}$, is lower bounded as:
\begin{align}\label{eq:CRB4}
 \textbf{{MSE}}(\hat{\boldsymbol{\theta}}_k^{\star}) \succeq \bigg(\sum_{n=1}^N\frac{\sigma^2}{\lambda^2[n]}(\mathbf{q}_k^T[n]\nabla_{\boldsymbol{\theta}_k}^{-1}\boldsymbol{\theta}_k^{\star})^T\mathbf{q}_k^T[n]\nabla_{\boldsymbol{\theta}_k}^{-1}\boldsymbol{\theta}_k^{\star}\bigg)^{-1},
\end{align}
where $\nabla_{\boldsymbol{\theta}_k}\boldsymbol{\theta}_k^{\star}$ is the Jacobian of $\boldsymbol{\theta}_k^{\star}$ with respect to $\boldsymbol{\theta}_k$.

\section{Numerical Evaluation}
\label{sec:Results}

We test the performance of the MLE in a system with $U=5$ DERs with the following generation capacities: $W_1=0.1\,\text{kW}$, $W_2=1\,\text{kW}$, $W_3=2\,\text{kW}$, $W_4=4\,\text{kW}$ and $W_5=15\,\text{kW}$.
The load components are fixed to the values $p_{cr}=3.5\,\text{kW}$, $p_{cc}=2.5\,\text{kW}$ and $p_{cp}=5\,\text{kW}$.
For brevity, we focus only on the MLE performed by controller $5$.
The rated voltage of the MG is $x=400\,\text{V}$ and $v_{min}=390\,\text{V}$.
The noise variance $\sigma^2$ is calculated as $\sigma^2 = \varphi^2((T_S-\tau)f_S)^{-1}$ where $\varphi$ is the sampling noise variance of the converter ADC.
In our evaluations $\varphi = 0.01\,\text{V/sample}$ \cite{ref22,ref23}, {$T_S-\tau=50\,\text{ms}$} and $f_S = 10\,\text{kHz}$.
{For illustration, we fix the number of slots to the lower limit $N=7$ which is necessary for Condition 1 to hold.}
{To ensure that Condition 1 is fully satisfied, we use binary orthogonal Hadamard training sequences with amplitude $\Delta x_u[n]\in\left\{-\delta x,+\delta x\right\},u=1,...,5$} and vary the value of $\delta\in[0.01\%,1\%]$.
The performance metric is the Relative Root Mean Squared Error.

Fig.~\ref{results} shows the performance of \eqref{eq:MLE_sol} for the generation capacities $W_u,u=1,2,3,4$ {as function of $\delta$}.
The low value of the estimator variance indicates that the DERs' generation capacities can be identified with practically negligible estimation error with quite short training sequences and relatively small voltage deviation amplitudes.
Fig.~\ref{results} also shows the performance of the load estimators.
Although the individual components of the load are identifiable and the estimators $\hat{p}_{cr}$, $\hat{p}_{cc}$ and $\hat{p}_{cp}$ are unbiased, the variance might be unsatisfactory for some applications.
The reason for this behavior lies in the fact that the load is a passive component and it does not transmit training sequences which makes the identification more uncertain.
However, as in many applications only the total load is of practical interest \cite{ref7,ref10}, Fig.~\ref{results} also depicts the performance of the estimator of the total load consumption $\hat{\omega}$.
Similarly to the case of the generation capacities, the results show promising performance since the controllers are capable to identify the load consumption with uncertainty lower than $0.1\%$ of the true value, with $N=7$ and {$\delta<0.5\%$ of the rated MG voltage $x$}.
Taking also into account the simplicity of \eqref{eq:MLE_sol}, these results highlight the great practical potential of the proposed approach.

 
\section{Ongoing and Future Work}
\label{sec:conc}

Our on-going work focuses on the analysis of the performance of the estimator for various different training protocols under practical system constraints on both the voltage ripple and the training length.
We are also working on extension of the approach for the practical case of multiple-bus DC/AC MG systems, as well as extension to other applications such as topology identification.
The CRB analysis shows highly encouraging results with promising practical outlook.

\vfill\pagebreak

\bibliographystyle{IEEEbib}
\bibliography{strings,refs}

\end{document}